\documentclass{ifacconf}

\usepackage{graphicx}      
\usepackage{natbib}        
\usepackage{amsmath}
\usepackage{amssymb}
\usepackage{xcolor}
\usepackage{xfrac}
\usepackage{algorithm}
\usepackage{algpseudocode}
\usepackage[T1]{fontenc}
\usepackage[makeroom]{cancel}

\newtheorem{define}{Definition}

\newtheorem{assumpt}{Assumption}

\newcommand{\setdef}[2]{  \left\{  #1\left| #2 \right. \right\} }
\newcommand{\setint}[1]{\textrm{int}\left( #1 \right)}
\begin{document}
\begin{frontmatter}

\title{Feasibility-aware Hybrid Control for Motion Planning under Signal Temporal Logics} 

\thanks[footnoteinfo]{
This work was supported by the ERC Consolidator Grant LEAFHOUND, the Swedish Research Council, Swedish Foundation for Strategic Research, and the Knut and Alice Wallenberg Foundation.
}

\author[KTH1]{Panagiotis Rousseas} 
\author[KTH1]{Dimos V. Dimarogonas} 

\address[KTH1]{KTH Royal Institute of Technology, School of Electrical Engineering and Computer Science, Division of Decision and Control Systems (e-mails: rousseas@kth.se, dimos@kth.se).}

\begin{abstract}                
In this work, a novel method for planar task and motion planning based on hybrid modeling is proposed. By virtue of a discrete variable which models local constraint satisfaction and enables local feasibility analysis, the proposed control architecture unifies planning with control design. Concurrently, control barrier functions are designed on a transformed disk version of the original nonconvex and geometrically complex robotic workspace, thus amending the issue of deadlocks. Simulations of the proposed method indicate effective handling of multiple overlapping spatio-temporal tasks even in the face of input saturation.  
\end{abstract}

\begin{keyword}
Event-based control, Task and motion planning, Control barrier functions and state space constraints, Spatio-temporal constraints, Hybrid and switched systems modeling
\end{keyword}

\end{frontmatter}

\section{Introduction}
Task planning problems have become increasingly relevant in recent years. Systems nowadays are not only required to carry out a prespecified series of tasks successfully, but importantly to plan for which tasks they may perform and when to execute them. This necessitates merging two fundamentally different paradigms, namely low-level continuous control with high-level discrete decision-making. Towards this direction, we propose a novel hybrid scheme where a continuous, simplified robot model is combined with a discrete variable that encodes which task-related constraints the robot is obeying at each time instance. Even though the robot model is simple, the system's workspace is non-convex while crucially, the proposed method is based on feasibility analysis that enables satisfying multiple overlapping constraints.    
\par 
Motion planning usually refers to finding kinodynamically feasible paths in the presence of obstacles, while task planning refers to designing a plan for executing a sequence of tasks, usually described through spatio-temporal constraints. Unifying the two through combined task and motion planning has been treated extensively in the literature (e.g., a comprehensive review is available by \cite{10705419}). An important aspect of task planning rests on formally representing tasks, where languages such as Linear Temporal Logic (LTL) and Signal Temporal Logic (STL) (\cite{stl_paper}) have been widely adopted. LTL and STL enable constructing metrics, e.g., see robustness metrics by \cite{10.1145/3501710.3519504,FAINEKOS20094262}, which can be used to quantify the degree of accomplishment of a task. In this context, numerous approaches have emerged, from sampling-based methods optimizing temporal robustness (\cite{Linard2023RealTimeRW}) or building spatiotemporal trees (\cite{10156470}), bezier curve-based planning (\cite{10644216}) and navigation function methods \citep{10.3389/frobt.2022.782783}.
\par 
An approach to satisfy STL constraints rests on forward-invariant-set methods, which have seen extensive development based on Control Barrier Functions (CBFs) (\cite{cbfqp}). CBFs can provide conditions for control synthesis that ensures forward invariance, and when designed correctly, such sets correspond to STL-task satisfaction (\cite{8404080}). Nevertheless, finding such CBFs is an open challenge for the general case, especially in the context of motion planning where nonconvex workspaces, internal obstacles and input bounds restrict robotic motion. As highlighted by \cite{8404080}, multiple STL constraints and input constraints may render control synthesis infeasible; hence feasibility analysis appears as necessary. Usually, multiple STL-CBFs are combined in a single CBF, without a-priori ensuring feasibility. Examples include (\cite{8404080}) as well as Model Predictive Control formulations (\cite{10421775}), where minimum violating solutions are proposed and prescribed performance control (\cite{10818715}) where the constraints can be relaxed to satisfy input bounds. Similar to our work, \cite{LINDEMANN2019284,8718798} propose a hybrid ``repair mechanism'' that also modifies STL-related specifications based on the system's input limitations.
\par 
An important aspect of the above works is that all of the STL-CBFs are combined into a single criterion; therefore, distinguishing between which tasks are actually performed at each time instance is ambiguous, which abstracts away the discrete aspect of task planning. This work aims at accounting for both aspects through a hybrid model, where an additional state variable is introduced to model exactly which tasks are attempted at a given time, by keeping STL-CBF constraints separate and handling conjunctions through a constraint-matrix formulation. This formulation lends itself to direct (local) feasibility analysis for the proposed CBF Quadratic Programming (CBF-QP) controller proposed by \cite{rousseas2025feasibilityevaluationquadraticprograms,lee2025constraintselectionoptimizationbasedcontrollers} enabling a search for feasible task execution. In terms of the motion planning aspect, we propose a novel CBF design that relies on a homeomorphic transformation of the robotic workspace to the disk (\cite{vlantis}), thus avoiding deadlocks that may appear through naive STL-CBF design in nonconvex workspaces. This bares resemblance to the navigation function approach by \cite{10.3389/frobt.2022.782783}, with the main advantage that we treat generic workspaces and not disk ones. The effectiveness of the proposed method is validated in simulations with and without input saturation and is compared to a constraint relaxation method, highlighting its ability for effectively tackling planar motion and task planning. 
\par 
In summary, the main contributions of this work are: 1) the definition of STL predicate functions on a transformed version of the initial planar robot workspace, which enables avoiding deadlocks due to workspace non-convexity; 2) a hybrid dynamical model that combines state feedback control with discrete decision making for task planning; and 3) a feasibility-informed control architecture that {\color{black}is able to detect mutually infeasible overlapping tasks and provides a decision-making mechanism to nevertheless satisfy them. This mechanism also improves the performance of the controller under input constraints.}
\section{Problem Formulation}
Consider a point robot, operating within a two-dimensional, connected planar workspace $\mathcal{G}\subset\mathbb{R}^2$ with $I\in\mathbb{N}$ inner distinct obstacles $\mathcal{O}_i \subset \mathcal{G},i\in\{1,\cdots,I\}\triangleq\mathcal{I}$. The robot's workspace is given by $\mathcal{W} = \mathcal{G}/\bigcup_{i\in\mathcal{I}}\mathcal{O}_i$ with boundary $\partial\mathcal{W} = \partial\mathcal{G}\bigcup_{i\in\mathcal{I}}\partial\mathcal{O}_i$. In this work, the single integrator dynamics are adopted\footnote{\color{black}Our method can be extended to more complex dynamical systems, e.g., double integrator, via lower level controllers or high-order CBFs.}:
\begin{equation}\label{eq:dynamics}
    \dot{p} = u(t),\quad  p(t = 0) = \bar{p}\in\setint{\mathcal{W}},
\end{equation}
where $p\in\setint{\mathcal{W}}$ denotes the robot's position, $\bar{p}$ denotes the robot's initial position and $u:\mathbb{R}_{\geq 0}\rightarrow\mathbb{R}^2$ denotes the robot's velocity and is considered as an input to be designed. We denote $p:\mathbb{R}_{\geq 0}\rightarrow\setint{\mathcal{W}}$ the solution to \eqref{eq:dynamics} (trajectory) from the initial condition $\bar{p}$ under the input signal $u$ with an abuse of notation, where the distinction between vector/time-signal will be clear in-context. A state-feedback form for $u$ is also considered through an abuse of notation as $u:\mathcal{W}\times\mathbb{R}_{\geq0}\rightarrow\mathbb{R}^2$, which can be viewed as a time signal through the solution of \eqref{eq:dynamics} under $u$ via the mapping $t\mapsto u\left( p(t),t\right)$ (the distinction will be clear in context).
Further consider that the input is constrained to belong inside a polytopic set $u(t) \in \mathcal{U}, \forall t \in \mathbb{R}_{\geq 0}$, where:
\begin{equation}\label{eq:input_constr}
    \mathcal{U} \triangleq \setdef{u\in\mathbb{R}^2}{A_s^\top u \geq B_s} \neq \emptyset,
\end{equation}
where $A_s\in\mathbb{R}^{2\times M}, B_s\in\mathbb{R}^M$ encode $M\in\mathbb{N}$ constraints such that $\mathcal{U}$ is nonempty. An admissible input is defined as follows:
\begin{define}\label{def:adm_ctrl}
    An input $u:\mathbb{R}_{\geq 0}\rightarrow\mathbb{R}^2$ is defined as admissible if $u(t)\in \mathcal{U}$, $\forall t\in\mathbb{R}_{\geq 0}$.
\end{define}
Consider now an STL (\cite{stl_paper}) predicate $\mu$ and its associated predicate function $h:\mathcal{W}\rightarrow\mathbb{R}$ as:
\begin{equation}\label{eq:pred_def}
    \mu :=
    \begin{cases}
        \mathrm{True}, &\ h(p)\geq 0\\
        \mathrm{False}, &\ h(p)< 0
    \end{cases}.
\end{equation}
The STL syntax defines an STL formula $\phi$ and is given by 
\[
\phi ::= \text{True} \mid \mu \mid \neg\phi \mid \phi_1 \land \phi_2 \mid \phi_1 \, U_{[a,b]} \, \phi_2 ,
\]
where $\phi_1, \phi_2$ are STL formulas and $a, b \in \mathbb{R}_{\ge 0}$ with $a \le b$.  
The satisfaction relation $(p,t) \models \phi$ denotes if the signal $p(t)$, satisfies $\phi$ at time $t$.
\begin{define}\label{def:stl}
For a signal $p : \mathbb{R}_{\ge 0} \rightarrow\mathcal{W}$, the STL semantics are recursively given by:
\[
\begin{aligned}
&(p,t) \models \mu 
\iff h(p(t)) \ge 0, \\
&(p,t) \models \neg\phi 
\iff \neg((p,t)\models\phi), \\
&(p,t) \models \phi_1 \land \phi_2 
\iff (p,t)\models\phi_1 \land (p,t)\models\phi_2, \\
&(p,t) \models \phi_1 U_{[a,b]} \phi_2 
\iff \exists t_1 \in [t+a, t+b] \; \\
&\qquad \text{s.t.}\; (p,t_1)\models\phi_2 \land \;\forall t_2 \in [t,t_1], (p,t_2)\models\phi_1, \\
&(p,t) \models F_{[a,b]}\phi 
\iff \exists t_1 \in [t+a, t+b] \;\text{s.t.}\; (p,t_1)\models\phi, \\
&(p,t) \models G_{[a,b]}\phi 
\iff \forall t_1 \in [t+a, t+b], (p,t_1)\models\phi.
\end{aligned}
\]
\end{define}
More specifically, $F_{[a,b]}$ is termed the ``eventually'' operator, $G_{[a,b]}$ is termed the ``always'' operator and $U_{[a,b]}$ is termed the ``until'' operator. In this work we consider these ``simple'' operators with their negations and importantly their conjunction (``and'' operator). Finally, we make the following assumption for the predicate functions $h$:
\begin{assumpt}\label{assum:pred_con}
    A predicate function $h:\mathcal{W}\rightarrow\mathbb{R}$ associated with a predicate $\mu$ \eqref{eq:pred_def} is twice continuously differentiable and its zero super-level set, defined as $\mathcal{H} = \setdef{p\in\mathcal{W}}{h(p)\geq 0}$ is closed and simply-connected.  
\end{assumpt}
\par 
\textbf{Problem:} Consider the robot model \eqref{eq:dynamics} operating within $\mathcal{W}$ and a collection of $J\in\mathbb{N}$ tasks encoded through STL predicates --see Def. \ref{def:stl}-- $\left\{\mu_1,\mu_2,\cdots,\mu_J \right\}$ and the corresponding set of predicate functions $\left\{ h_1,h_2,\cdots,h_J\right\}, h_j:\mathcal{W}\rightarrow \mathbb{R}, \forall j\in\left\{1,\cdots,J\right\}$. The goal is to formulate an \textbf{admissible} control law per Def. \ref{def:adm_ctrl}, i.e., a mapping $(p,t)\mapsto k\left( p,t\right)\in\mathcal{U}$ that satisfies as many of the aforementioned STL-encoded tasks as possible for any $p(0) = \bar{p}\in\mathcal{W}$. 


\section{Proposed Method}\label{sec:prop_meth}
\subsection{Harmonic Maps}
In order to amend deadlocks that may arise inside non-convex workspaces, we employ a harmonic, homeomorphic map (\cite{vlantis}), denoted as $\mathcal{T}:\mathcal{W}\rightarrow\mathcal{D}$ that transforms an arbitrary multiply connected workspace $\mathcal{W}$ into the punctured unit disk $\mathcal{D} \subseteq \setdef{x\in\mathbb{R}^2}{\|x\|\leq 1}$:
\begin{equation}\label{eq:harm_tf}
    \mathcal{D} \ni q = \mathcal{T}(p),\ \forall p \in \setint{\mathcal{W}}.
\end{equation}
The outer boundary $\partial\mathcal{G}$ is mapped onto the unit circle, while the inner obstacles are mapped to points inside $\mathcal{D}$. Furthermore, the dynamics \eqref{eq:dynamics} can be expressed in $\mathcal{D}$ through the transformation's Jacobian: {\color{black}$\dot{q} = \mathcal{J}(p)\dot{p} = \mathcal{J}(p)u \triangleq v$,}
where $\mathcal{J} = \left[ \tfrac{\partial q_i}{\partial p_j}\right]:\mathcal{W}\mapsto \mathbb{R}^{2\times 2},i,j\in\{1,2\}$ is always full rank since $\mathcal{T}$ is a homeomorphism and $v\in\mathbb{R}^2$ is a virtual input. This enables designing control laws with desired properties for the virtual input $v$ on the disk, and mapping the virtual input back to the robot's velocity via the inverse of the Jacobian (which is always invertible owing to $\mathcal{T}$ being a homeomorphism): 
\begin{equation}\label{eq:v_to_u}
    u = \mathcal{J}^{-1}(p)v.
\end{equation}
This homeomorphism is employed to construct STL-based CBFs for control design in non-convex workspaces. First, given a predicate $h:\mathcal{W}\rightarrow\mathbb{R}$ consider the mapping of its zero super-level set onto the disk:
\begin{equation}\label{eq:cbf_tf_zr_lvl_set}
    \mathcal{H}_T \triangleq \setdef{q\in\mathcal{D}}{ h \left( \mathcal{T}^{-1}(q) \right) \geq 0},
\end{equation}
where $ \mathcal{T}^{-1}:\mathcal{D} \rightarrow\mathcal{W}$ denotes the inverse of $\mathcal{T}$. Since $\mathcal{T}$ is a homeomorphism, and since by Assum. \ref{assum:pred_con} $\mathcal{H}$ is closed and connected, $\mathcal{H}_T $ is also closed and connected. Therefore, we propose the following transformed predicate $h_T:\mathcal{D}\rightarrow\mathbb{R}$ defined on the disk: 
\begin{equation}\label{eq:pred_disk}
\color{black}
    h_T(q) = 
    \begin{cases}
        -\underset{\tilde{q} \in \partial \mathcal{H}_T}{\min} 
        \left\{ 
            \|q - \tilde{q}\|
        \right\}, & q\notin \mathcal{H}_T \\
        +\underset{\tilde{q} \in \partial \mathcal{H}_T}{\min} 
        \left\{ 
            \|q - \tilde{q}\|
        \right\}, & q\in \mathcal{H}_T
    \end{cases}.
\end{equation}
Importantly for the sequel, note that closure and connectivity of $\mathcal{H},\mathcal{H}_T$ imply that for $p \in \mathcal{W},q\in \mathcal{D}$ such that {\color{black}$q = \mathcal{T}(p)$}: $p \in \mathcal{H} \iff q\in\mathcal{H}_T$, or equivalently, $h(p)\geq 0 \iff h_T(q)\geq 0$. Therefore, in order to satisfy the associated predicate $\mu$ on $\mathcal{W}$, it suffices to ensure that $h_T\geq 0$, which enables control design on $\mathcal{D}$, avoiding deadlocks due to non-convexity of $\mathcal{W}$.

\subsection{STL-based CBFs}\label{sec:stl_cbfs}
Given two time instances $t_0,t_1 \in \mathbb{R}_{\geq 0}$ such that $t_0<t_1$, consider the differentiable function $b:\mathcal{D}\times [t_0,t_1] \rightarrow\mathbb{R}$ and the set:
\begin{equation}\label{eq:cbf_set}
    \mathcal{B}(t) \triangleq \left\{q\in\mathcal{D}|b(q,t)\geq 0 \right\}.
\end{equation}
\begin{define}
A differentiable function ${b} : \mathcal{D} \times [t_0, t_1] \to \mathbb{R}$ is a candidate CBF if for each
${q}_0 \in \mathcal{B}(t_0)$, there exists an absolutely continuous function
${q} : [t_0, t_1] \to \mathbb{R}^n$ with ${q}(t_0) := {q}_0$ such that
${q}(t) \in \mathcal{B}(t)$ for all $t \in [t_0, t_1]$.
\end{define} 
According to \cite{8404080}, candidate CBFs can be designed based on STL predicates \eqref{eq:pred_def} to ensure the satisfaction of the latter. Similar to previous works by \cite{wiltz2025uniformlytimevaryingcontrolbarrier}, the CBFs assume the following general form: \( b\left( q,t\right) = \sigma(t) \left( h_T(q) + \gamma\left( z,t \right) \right),\)
where $h_T$ denotes a transformed STL predicate \eqref{eq:pred_disk}, the function $\gamma:\mathcal{D}\times \mathbb{R}_{\geq 0}\rightarrow\mathbb{R}$ is a state and time-dependent modifier designed such that forward invariance of the STL-CBF set \eqref{eq:cbf_set} guarantees satisfaction of the corresponding predicate and $\sigma:\mathbb{R}_{\geq 0}\rightarrow[0,1]$ is a switching function to be defined in the sequel. The argument $z\in\mathcal{D}$ is an additional parameter that is designed to ensure forward invariance of the predicate super-level set. In this work we propose the following STL-CBFs:

\subsubsection{\textbf{Eventually Operator}}
The STL-CBF for the eventually operator $F_{[t_0,t_1]}\mu$ is denoted by $b^F_{[t_0,t_1]}:\mathcal{D}\times\mathbb{R}_{\geq 0}\rightarrow \mathbb{R}$ and is designed as:
\begin{equation}\label{eq:eventually_cbf}
    b^F_{[t_0,t_1]}(q,t) \triangleq \sigma^{t_0,t_1}_{\delta}(t) 
    h_T(q),
\end{equation}
where $\sigma^{t_0,t_1}_{\delta}:\mathbb{R}_{\geq 0}\rightarrow[0,1]$ for some $0<\delta<t_1-t_0$ is given by --see Fig. \ref{fig:sigma_01}--:
\begin{equation}\label{eq:step_fun_0_1}
    \sigma^{t_0,t_1}_{\delta}(t) \triangleq \sigma_{t_1,\delta}(t)\left( 1 - \sigma_{t_0,\delta}(t)\right), 
\end{equation}
for $\sigma_{\tau,\delta}:\mathbb{R}_{\geq 0}\rightarrow [0,1]$:
\begin{equation}\label{eq:step_fun_0}
    \sigma_{\tau,\delta}(t) = 
    \begin{cases}
        1,&\quad t \leq \tau  - \delta \\
        \exp
            \left( 
                -\left[
                    \tfrac{t - \left( \tau - \delta \right)}{t - \tau}
                \right]^2
            \right),&\quad \tau - \delta < t < \tau \\
        0, &\quad t \geq \tau
    \end{cases},\notag 
\end{equation}
\begin{figure}
    \centering
    \includegraphics[trim= 0.95cm 0.cm 1cm 0cm,clip,width=0.8\linewidth]{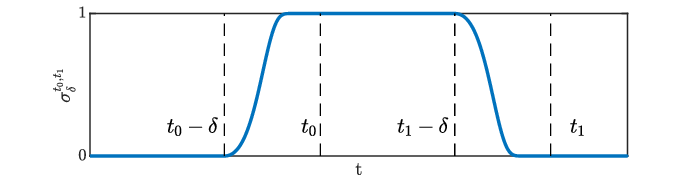}
    \caption{Switching function \eqref{eq:step_fun_0_1}.}
    \label{fig:sigma_01}
\end{figure}
where $\sigma_{\tau,\delta}$ serves as a smooth switching function that appropriately nullifies the corresponding STL-CBFs based on their time-window. 

\subsubsection{\textbf{Always Operator}}
The STL-CBF for the always operator $G_{[t_0,t_1]}\mu$ is denoted by $b^G_{[t_0,t_1]}:\mathcal{D}\times\mathbb{R}_{\geq 0}\rightarrow \mathbb{R}$ and is designed as:
\begin{equation}\label{eq:always_cbf}
    b^G_{[t_0,t_1]}(q,t) \triangleq 
    \sigma_{t_1,\delta}(t) \cdot  
    \left(
        h_T(q) + \gamma^G_\delta(z,t)
    \right),
\end{equation}
for $z = \mathcal{T}(\hat{p})$, for some $\hat{p}\in\mathcal{W}$, where $\gamma^G_\delta:\mathcal{D}\times \mathbb{R}_{\geq 0}\rightarrow\mathbb{R}$ is --See Fig. \ref{fig:gamma_G}--: $\gamma^G_\delta(z,t) =
         a \cdot\exp\left(   \frac{t}{t_0}\ln\left( 1 + \frac{h_T(z)}{a}\right)\right) 
        - \left( a  + h_T(z)\right)$,
and $a \in \mathbb{R}_{\geq 0}$ is chosen such that: $0 < 1 + \tfrac{h_T(z)}{a} < 1$ (e.g., $a = -2 h_T(z)$). Additionally, $z$ can be chosen as the initial transformed position $\bar{q} = \mathcal{T}(\bar{p})$, (we remind the reader that $\bar{p}$ is the initial robot position --see Eq. \eqref{eq:dynamics}--) but other choices will be investigated in the future.  
\begin{figure}
    \centering
    \includegraphics[trim= 1.2cm 0cm 1.2cm 0cm,clip,width=0.8\linewidth]{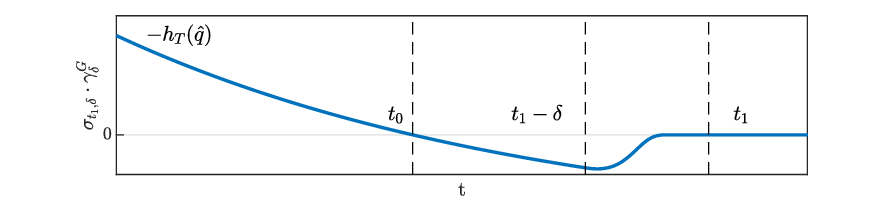}
    \caption{Function $\sigma_{t_1,\delta}(t) \cdot\gamma^G_\delta(z,t)$ for $a = -2 h_T(z)$.}
    \label{fig:gamma_G}
\end{figure}

\subsubsection{\textbf{Until Operator}}
The ``until'' $\mu_1 U_{[t_0,t_1]}\mu_2$ operator can be synthesized by combining an always and an eventually operator --see \cite{8404080}--.
\par 
{\color{black}We are now ready to prove that ensuring forward invariance of the zero super-level sets \eqref{eq:cbf_set} of the above STL-based CBFs \eqref{eq:eventually_cbf}, \eqref{eq:always_cbf}, ensures satisfaction of the corresponding predicates.}

\begin{thm}\label{thm:until}
    Given a predicate $\mu$, its corresponding predicate function $h:\mathcal{W}\rightarrow\mathbb{R}$, the eventually operator $F_{[t_0,t_1]}\mu$ and the STL-CBF \eqref{eq:eventually_cbf}, if $\exists u \in \mathcal{U}$ such that $\mathcal{B}^F(t) \triangleq \left\{q\in\mathcal{D}|b^F_{[t_0,t_1]}(q,t)\geq 0 \right\}$ is forward invariant, then the corresponding predicate is satisfied. 
\end{thm} 
\begin{pf}
    Let $\bar{p}\in\mathcal{W}$ denote the initial position of the robot, with its transformed equivalent $\bar{q} = \mathcal{T}(\bar{p})$. Then, $\forall \tau \in [0,t_0 - \delta]$:
    \(
            b^F_{[t_0,t_1]}\left( q(\tau),\tau \right) 
            = \sigma^{t_0,t_1}_{\delta}(\tau)
            h_T\left(  q(\tau) \right)= 0,
    \)
     since $\sigma^{t_0,t_1}_{\delta}(\tau) = 0, \forall \tau \in [0,t_0 - \delta]$ (see Fig. \ref{fig:sigma_01}), which implies that $\bar{p} = p(0) \in \mathcal{B}^F(0)$. Furthermore, $\forall \tau \in [t_0,t_1-\delta], b^F_{[t_0,t_1]}\left( q(\tau),\tau \right) =  h_T(q(\tau))$, hence forward invariance of $\mathcal{B}^F(t)$, implies that $h_T(q(\tau)) \geq 0$ for some $\tau \in [t_0,t_1-\delta]$, which in turn implies that $h(p(\tau))\geq 0$, owing to the fact that $\mathcal{T}(\cdot)$ is a homeomorphism, concluding the proof.  
\end{pf}

\begin{thm} \label{thm:always}
    Given a predicate $\mu$, its corresponding predicate function $h:\mathcal{W}\rightarrow\mathbb{R}$, the always operator $G_{[t_0,t_1]}\mu$ and the STL-CBF \eqref{eq:always_cbf} for $z = \mathcal{T}(\bar{p}) = \mathcal{T}({p}(t=0))$, if $\exists u \in \mathcal{U}$ such that $\mathcal{B}^G(t) \triangleq \left\{q\in\mathcal{D}|b^G_{[t_0,t_1]}(q,t)\geq 0 \right\}$ is forward invariant, then the predicate is satisfied. 
\end{thm} 
\begin{pf}
    Let $\bar{p}\in\mathcal{W}$ denote the initial position of the robot, with its transformed equivalent $\bar{q} = \mathcal{T}(\bar{p}) = q(t=0)$. Then, for $t = 0$:
    \begin{equation}
    \begin{split}
            b^G_{[t_0,t_1]}\left( q(0),0 \right) \overset{\eqref{eq:always_cbf}}{=} h_T(q(0))
        - \left( a  + h_T(\bar{q})\right)  +a \cdot\exp\left( 0\right) =0
    \end{split}\notag 
    \end{equation}
    since $\sigma^{t_0,t_1}_{\delta}(\tau) = 1, \forall \tau \in [0,t_0 - \delta]$ (see Fig. \ref{fig:sigma_01}). The above implies that $\bar{q} = q(0) \in \mathcal{B}^G(0)$. Then, $\forall \tau \in [0,t_1]$: 
    {\color{black}\( b^G_{[t_0,t_1]}\left( q(\tau),\tau \right) \geq 0  \overset{\sigma^{t_0,t_1}_{\delta}\geq 0}{\iff} h_T(q(\tau)) + \gamma^G_\delta(z,\tau) \geq 0 \Leftrightarrow 
            h_T(q(\tau)) \geq - \gamma^G_\delta(z,\tau) \geq 0,\)}
     since $\gamma^G_\delta(z,\tau) \leq 0$ by construction $\forall \tau \geq t_0$. Therefore, forward invariance of $\mathcal{B}^G(t)$, implies that $h_T(q(\tau)) \geq 0, \forall \tau \in [t_0,t_1]$, which in turn implies that $h(p(\tau))\geq 0$, owing to the fact that $\mathcal{T}(\cdot)$ is a homeomorphism, concluding the proof.  
\end{pf}

Since the CBFs for the eventually and always operators have been verified, the until operator is also shown to be valid by composing an eventually with an always STL-CBF. Consider the STL-CBF under-approximation of the minimum operator in \cite[Subsec. IIIA]{8404080}. Then, using Thms. \ref{thm:until}, \ref{thm:always} the zero-superlevel set of the former can be shown to be forward time invariant. The proof is omitted for brevity.

\subsection{CBF-based Control Synthesis}
Given a time-dependent CBF and its corresponding zero super-level set \eqref{eq:cbf_set} the latter is shown to be forward-time invariant over the signal (trajectory) $q:\mathbb{R}_{\geq 0}\rightarrow\mathcal{D}$, iff --see \cite{wiltz2025uniformlytimevaryingcontrolbarrier}-- \(\dot{b}(q(t),t) \geq - \alpha\left( b(q(t),t) \right),\)
for some class-$\mathcal{K}$ function $\alpha:\mathbb{R}\rightarrow\mathbb{R}$. Given \eqref{eq:dynamics}, the CBF condition becomes:
\begin{equation}\label{eq:forw_inv:2}
\begin{split}
    \left. \frac{\partial b}{\partial t} \right|_{(q,t)} + 
     \nabla_q^\top b (q,t)\dot{q} \geq  - \alpha\left( b(q,t) \right).
\end{split}
\end{equation}
Consider now a collection of ``simple'' STL predicates of Prob. 1 $\left\{\mu_1,\mu_2,\cdots,\mu_J \right\}$, the corresponding set of predicate functions $\left\{ h_1,h_2,\cdots,h_J\right\}$, and their associated CBF designed as in Sec. \ref{sec:stl_cbfs} denoted as $\{b_1,b_2,\cdots,b_J\}, b_j:\mathcal{D}\times\mathbb{R}_{\geq 0}\rightarrow\mathbb{R}, \forall j \in \{1,2,\cdots,J\}$. In order to enforce the conjunction of the above, \eqref{eq:forw_inv:2} yields:
\begin{equation}\label{eq:forw_inv:mult}
    \begin{split}
        A^\top(q,t) v
        &\geq 
        B(q,t),
    \end{split}
\end{equation}
where:
\begin{equation}\label{eq:A_matrix}
    \begin{split}
    A(q,t) = 
        \begin{bmatrix}
            \nabla_q b_1 (q,t),&
            \cdots, &
            \nabla_q b_J (q,t)
        \end{bmatrix},
    \end{split}
\end{equation}
\begin{equation}\label{eq:B_matrix}
    \begin{split}
    B(q,t) =  
        -\begin{bmatrix}
             \alpha\left( b_1 \left( q,t \right)\right)
            +\sfrac{\partial b_1}{\partial t}\left( q,t \right) \\
            \vdots \\
            \alpha\left( b_J\left( q,t \right) \right)
            +\sfrac{\partial b_J}{\partial t} \left( q,t \right)
        \end{bmatrix}.
    \end{split}
\end{equation}
Through \eqref{eq:v_to_u}, Eq. \eqref{eq:forw_inv:mult} becomes:
\begin{equation}\label{eq:forw_inv:mult:2}
    A^\top (q,t)\mathcal{J}(p)u \geq B(q,t).
\end{equation}
Finally, a \textbf{feasible} input to \eqref{eq:dynamics}, denoted by $k:\mathcal{W}\times \mathbb{R}_{\geq 0}$, that satisfies the predicates $\left\{\mu_1,\mu_2,\cdots,\mu_J \right\}$ can be computed via the solution of a Quadratic Program (QP):
\begin{equation}\label{eq:qp_con}
\begin{gathered}
    k(p,t) = \underset{u\in\mathbb{R}^m}{\arg \min} 
    \left\{ 
        \|u\|^2
    \right\}, \\
    \mathrm{s.t. } 
    \begin{bmatrix}
        A_s^\top \\
        A^\top\left(\mathcal{T}(p),t \right)\mathcal{J}(p)
    \end{bmatrix}u
    \geq 
    \begin{bmatrix}
        B_s \\
        B\left(\mathcal{T}(p),t \right)
    \end{bmatrix},
\end{gathered}
\end{equation}  
where $A_s,B_s$ are given in \eqref{eq:input_constr}.

\subsection{Hybrid Control and Compatibility Check}
The conjunction of multiple STL-CBFs may render \eqref{eq:qp_con} infeasible, which necessitates a planning framework that selects for a subset of the former to be enforced at each time instance. The notion of a \textbf{configuration vector} was proposed by \cite{rousseas2025feasibilityevaluationquadraticprograms}. A configuration is defined as a vector $C\in\{0,1\}^J\triangleq\mathcal{C}$, where a $1$ (resp. $0$) entry in the $j$-th place of $C$ corresponds to the $j$-th STL-CBF constraint --encoded in the $j$-th row of \eqref{eq:forw_inv:mult:2}-- being enforced (resp. dropped). To model this, consider the following modification to \eqref{eq:qp_con}:
\begin{equation}\label{eq:qp_con*}
\begin{gathered}
    \bar{k}(p,t,C) = \underset{u\in\mathbb{R}^m}{\arg \min} 
    \left\{ 
        \|u\|^2
    \right\}, \\
    \mathrm{s.t. } 
    \begin{bmatrix}
        A_s^\top \\
        \textrm{diag}\left(C\right) A^\top\left(\mathcal{T}(p),t \right)\mathcal{J}(p)
    \end{bmatrix}u
    \geq 
    \begin{bmatrix}
        B_s \\
        \textrm{diag}\left(C\right) B\left(\mathcal{T}(p),t \right)
    \end{bmatrix},
\end{gathered}
\tag{\ref{eq:qp_con}*}
\end{equation} 
where $\bar{k}:\mathcal{W}\times \mathbb{R}_{\geq 0}\times \mathcal{C} \rightarrow \mathcal{U}$ and $\textrm{diag}(C)$ is the matrix with $C$'s elements on its diagonal and zeros elsewhere, which serves to nullify the disregarded constraints, as encoded in the configuration vector $C\in\mathcal{C}$.

In order to assess feasibility of \eqref{eq:qp_con}, we employ the point-wise feasibility check by \cite{rousseas2025feasibilityevaluationquadraticprograms}, where it is proven that \eqref{eq:qp_con*} is feasible iff the following Linear Program (LP), denoted by $\mathcal{P}(p,t,C)$ for a given position $p\in\mathcal{W}$, time instance $t\in\mathbb{R}_{\geq 0}$ and configuration $C\in\mathcal{C}$, admits a bounded solution:
\begin{equation}\label{eq:feas_lp}
\begin{gathered}
\mathcal{P}(p,t,C) :
    \underset{\lambda \in \mathbb{R}^{M+J}}{\max}
    \left\{ 
        \begin{bmatrix}
            B_s^\top, &
             B^\top\left(\mathcal{T}(p),t \right)\textrm{diag}\left(C\right)
    \end{bmatrix}\lambda 
    \right\}, \\
    \mathrm{s.t.: }
    \qquad \lambda \geq \mathbf{0}_{(M+J)\times 1}, \\
        \begin{bmatrix}
            A_s^\top, \\
             \textrm{diag}\left(C\right) A^\top\left(\mathcal{T}(p),t \right)\mathcal{J}(p)
        \end{bmatrix}^\top \lambda = \mathbf{0}_{(M+J)\times 1},
\end{gathered}\notag 
\end{equation}
and the following point-wise feasibility check is defined $\mathcal{F}:\mathcal{W}\times\mathbb{R}_{\geq 0}\times \mathcal{C} \rightarrow\{0,1\}$ (\cite{rousseas2025feasibilityevaluationquadraticprograms}):
\begin{equation}\label{eq:feas_check}
    \mathcal{F}\left(p,t,C\right)
    \triangleq
    \begin{cases}
        1, & \textrm{ if } \mathcal{P}(p,t,C) \textrm{ is bounded} \\
        0, & \textrm{ if } \mathcal{P}(p,t,C) \textrm{ is unbounded}
    \end{cases}
\end{equation}
Now, consider the following hybrid system for \eqref{eq:dynamics}:
\begin{equation}\label{eq:dynamics_hybrid}
    \mathcal{S}
    \begin{cases}
        \begin{bmatrix}
            \dot{p} \\
            \dot{C}
        \end{bmatrix} = 
        \begin{bmatrix}
            \bar{k}\left( p,t,C \right)\\
            \mathbf{0}_{J\times 1}
        \end{bmatrix}, 
        &
        \left\{
            p \in \mathcal{W}, C\in \mathcal{C} \left| \mathcal{F}\left(p,t,C\right) = 1\right.
        \right\}
        \\[1.5em]
         \begin{bmatrix}
            {p} ^+\\
            {C}^+
        \end{bmatrix} = 
        \begin{bmatrix}
            p\\
            V
        \end{bmatrix},
        &
         \left\{
            p \in \mathcal{W}, C\in \mathcal{C} \left| \mathcal{F}\left(p,t,C\right) = 0\right.
        \right\}
    \end{cases},
\end{equation}
where $u\in\mathcal{U}$ and $V\in\mathcal{C}$ denotes a virtual input that models configuration selection, i.e., which of the CBF-STL constraints are enforced at position $p\in\mathcal{W}$ and time $t\in\mathbb{R}_{\geq 0}$ if \eqref{eq:qp_con*} is infeasible. The virtual input $V \in \mathcal{C}$ should be designed such that $\mathcal{F}\left(p,t,C^+ \right) = 1$ and such a process is outlined in Subsec. \ref{sec:subsec:conf_search}. 

\subsection{Configuration Search}\label{sec:subsec:conf_search}
The main challenge for controlling system \eqref{eq:dynamics_hybrid} rests on designing the virtual input $V\in\mathcal{C}$ for any $p\in\mathcal{W},t\in\mathbb{R}_{\geq 0}$, i.e., selecting among configurations that render Prob. \eqref{eq:qp_con*} feasible (termed feasible configurations). Note that the configuration space $\mathcal{C}$ is a discrete space with cardinality $|\mathcal{C}| = 2^J$, hence the problem of designing $V\in\mathcal{C}$ boils down to searching over the $2^J$ possible configurations for a feasible one. For instance, one criterion for selecting among feasible configurations, is finding one that maximizes the number of enforced constraints. This is termed the maxFS problem, and is NP-hard --see \cite{Chinneck02102019}--. 
\par 
In this work, a heuristic search is employed based on task prioritization --that may reflect operational, functional, etc. requirements-- where constraints with higher priority are more likely to be enforced. If such a prioritization exists for the predicates $\{\mu_1,\mu_2,\cdots,\mu_J\}$ the corresponding rows and elements of the matrix $A$ \eqref{eq:A_matrix} and the vector $B$ \eqref{eq:B_matrix} in \eqref{eq:forw_inv:mult:2} are ranked based on this priority. For instance, in the absence a pre-specified task ranking,
a reasonable ranking is based on the final time for each predicate; tasks that need to be completed earlier are prioritized over tasks that may elapse later on. This is evidently not a solution to the masFS problem, nor does it guarantee satisfaction of all of the tasks, especially in face of input constraints \eqref{eq:input_constr}. A thorough analysis on such search methods is an open problem and is considered outside the scope of this paper.   
\par 
The proposed search method is outlined in Alg. \ref{alg:conf_search} and is as follows: starting from a configuration $C  = \left[ c _1,c _2,\cdots,c _J\right]^\top\in\mathcal{C}$ we attach a binary number $B  = c_1c_2\cdots c_J$ to the former. Then we search over configurations by iteratively subtracting $1$ from $B $. This has the effect that the least important constraints (corresponding to the right-most digits of $B $) are disregarded first.  
\begin{algorithm}
\caption{{\fontfamily{qcr}\selectfont Configuration Search} }\label{alg:conf_search}
\begin{algorithmic}
\State $\bullet$ \textbf{Input:} configuration $C= \left[ c _1,c _2,\cdots,c _J\right]\in\mathcal{C}$,
\State Given $p\in\mathcal{W},t\in\mathbb{R}_{\geq 0}$,
\State $\bullet$ {\fontfamily{qcr}\selectfont feasible} $\gets$ False,
\While{ \textbf{NOT}({\fontfamily{qcr}\selectfont feasible}) and  \textbf{NOT}$(C = \mathbf{0}_{J\times 1})$}
    \State $\bullet$ Create binary $B  = c _1c _2\cdots c _J$,
    \State $\bullet$ Binary subtraction 
    \(
        B' = B - \underbrace{00\cdots 1}_{J \textrm{ digits}} = c'_1c'_2\cdots c'_J,
    \)
    \State $\bullet$ Create new configuration   
        \( C' = \left[c'_1,c'_2,\cdots,c'_J\right]^\top \),
    \If{$j$-th constraint with $c_j=1$ has not elapsed/been satisfied, $\forall j \in\{1,\cdots,J\}$}
        \State $\bullet$ Feasibility check:
        \begin{equation}
            \textrm{{\fontfamily{qcr}\selectfont feasible}}
            \gets 
            \begin{cases}
                \textrm{False},& \mathcal{F}(p,t,C') = 0 \\
                \textrm{True}, & \mathcal{F}(p,t,C') = 1 
            \end{cases},\notag 
        \end{equation}
    \EndIf
    \State $\bullet$ $C \gets C'$,
\EndWhile
 \State $\bullet$ Return $C$
\end{algorithmic}
\end{algorithm}
Since the proposed serach is ``top-down'', i.e., starting from many enforced constraints and removing them one-by-one, at worst, Alg. \ref{alg:conf_search} will search over all $2^J$ possible configurations. However, this is known to be a combinatorial, NP-hard problem and is the subject of active research (\cite{lee2025constraintselectionoptimizationbasedcontrollers}). The main issue is that, while the configuration $C = \mathbf{0}_{J\times 1}$ is feasible by assumption, a large number of constraints may be disregarded to ensure feasibility, thus requiring many iterations of Alg. \ref{alg:conf_search}. Therefore, a more conservative, but faster approach, which searches ``bottom-up'' may start from $C = \mathbf{0}_{J\times 1}$ and iteratively add constraints until some time interval elapses or a number of feasible constraints are enforced. Furthermore, Alg. \ref{alg:conf_search} is parallelizable. We leave this as part of future work.  

\subsection{Overall Control Scheme}
In this subsection, the control architecture that incorporates the aforementioned elements is outlined in Alg. \ref{alg:control_scheme}. Specifically, the algorithm begins with transforming the workspace and the STL predicate sets onto the disk domain, followed by designing the STL-CBFs on the disk. Then, control of system \eqref{eq:dynamics} proceeds through the CBF-QP controller \eqref{eq:qp_con*} for the hybrid system \eqref{eq:dynamics_hybrid} until either the problem becomes infeasible under the current configuration or until some of the predicates are satisfied/have elapsed. In case infeasibility is detected, the {\fontfamily{qcr}\selectfont Configuration Search} algorithm is employed from the current \textit{infeasible} configuration. Otherwise, if some predicate is satisfied, the {\fontfamily{qcr}\selectfont Configuration Search} algorithm is called to search the full space of configurations, i.e., from the configuration $C = \mathbf{1}_{J \times 1}$, in an effort to enforce as many STL-based constraints as possible. 

We underline that herein, infeasibility is interpreted point-wise, i.e., the enforced constraints/tasks in Eq. \eqref{eq:forw_inv:mult:2}, given by the configuration $C$ in \eqref{eq:forw_inv:mult:2} (which may include all, some or none of the tasks), cannot be simultaneously satisfied. The selected configuration per Alg. \ref{alg:conf_search} ensures that the input constraints are not violated, while satisfying tasks that are both collectively feasible and obey input constraints.  
\begin{algorithm}
\caption{Control Scheme}\label{alg:control_scheme}
\begin{algorithmic}
    \State $\bullet$ Given:
    \begin{enumerate}
        \item  A workspace $\mathcal{W}$ and an initial position $\bar{p}\in\mathcal{W}$, 
        \item  the collection of ``simple'' STL predicates of Prob. 1 $\left\{\mu_1,\mu_2,\cdots,\mu_J \right\}$, 
        \item  the corresponding set of predicate functions $\left\{ h_1,h_2,\cdots,h_J\right\}$, 
        \item  their associated zero super-level sets $\left\{ \mathcal{H}^1, \mathcal{H}^2,\cdots, \mathcal{H}^J \right\}$ \eqref{eq:cbf_tf_zr_lvl_set}, 
    \end{enumerate}
    \State $\bullet$ Obtain the disk transformation \eqref{eq:harm_tf},
    \State $\bullet$ Obtain the STL-CBF boundaries on $\mathcal{D}$, denoted by $\left\{ \partial \mathcal{H}_T^1, \partial \mathcal{H}_T^2,\cdots, \partial \mathcal{H}_T^J \right\}$,
    \State $\bullet$ Design the associated CBFs as in Sec. \ref{sec:stl_cbfs} denoted as $\{b_1,b_2,\cdots,b_J\}, b_j:\mathcal{D}\times\mathbb{R}_{\geq 0}\rightarrow\mathbb{R}, \forall j \in \{1,2,\cdots,J\}$, ranked based on the final times $\{ t_1^1,t_1^2,\cdots,t_1^J\}$ such that $t_1^1 \leq t_1^2 \leq\cdots\leq t_1^J$.  
    \State $\bullet$ {\fontfamily{qcr}\selectfont {\color{black}satsfd}} $\gets$ False,
    \State $\bullet$ {\fontfamily{qcr}\selectfont compatible} $\gets$ False,
    \State $\bullet$ {\fontfamily{qcr}\selectfont {\color{black}satsfd} tasks} $\gets$ 0,
    \State $\bullet$ Starting from $t=0$, $C = \mathbf{1}_{J \times 1}$, 
\While{$t\leq t_1^J$ \textbf{AND} {\fontfamily{qcr}\selectfont ({\color{black}satsfd} tasks)}$<J$}
    \While{{\fontfamily{qcr}\selectfont (compatible)} \textbf{AND} \textbf{NOT}({\fontfamily{qcr}\selectfont {\color{black}satsfd}})}
        \State $\bullet$ Compute matrices in \eqref{eq:forw_inv:mult:2} through \eqref{eq:A_matrix}, \eqref{eq:B_matrix},
        \State $\bullet$ Advance system \eqref{eq:dynamics_hybrid}, through the input \eqref{eq:qp_con*},
             \If{
                \(
                    \mathcal{F}\left(p(t),t,C\right) = 0,
                \)
                }
                \State $\bullet$ {\fontfamily{qcr}\selectfont compatible} $\gets$ False,
            \EndIf
        \State $\bullet$ Check if one of the predicates elapses or is satisfied (resp. has not elapsed, is not satisfied) and set {\fontfamily{qcr}\selectfont {\color{black}satsfd}} to True (resp. False),
    \EndWhile
    \If{\textbf{NOT}({\fontfamily{qcr}\selectfont compatible})}
        \State $\bullet$ Search for compatible configurations {\fontfamily{qcr}\selectfont Configuration Search} from configuration $C$
        \State $\bullet$ {\fontfamily{qcr}\selectfont compatible} $\gets$ True,
    \ElsIf{{\fontfamily{qcr}\selectfont {\color{black}satsfd}}}
        \State $\bullet$ Search for compatible configurations {\fontfamily{qcr}\selectfont Configuration Search} from configuration $C = \mathbf{1}_{J\times 1}$,
        \State $\bullet$ {\fontfamily{qcr}\selectfont {\color{black}satsfd}} $\gets$ False,
        \State $\bullet$ {\fontfamily{qcr}\selectfont {\color{black}satsfd} tasks} $+= 1$,
    \EndIf
\EndWhile
\end{algorithmic}
\end{algorithm}

\section{Simulations}
The simulations were programmed in MATLAB2024a on a laptop running Windows 10 and equipped with an Intel Core Ultra 7 165U CPU. We present two examples, one for unsaturated and one for saturated inputs to \eqref{eq:dynamics}, to drive home the effect of input constraints in satisfying STL predicates. We consider the workspace of Fig. \ref{fig:ws:unsat}, {\color{black}and the formula with 8 predicates:
\begin{equation}\label{eq:stl:formula}
    \begin{split}
        &\phi = 
        G_{[50,300]}\wedge G_{[500,750]} \wedge F_{[200,500]}\wedge F_{[200,1600]}
        \\
        &\wedge F_{[400,2100]}\wedge F_{[1800,2200]}\wedge F_{[500,2800]}\wedge F_{[2500,3500]}.
    \end{split}
\end{equation}
}
For our simulations, all predicates require that the robot reaches circular regions, whose predicate functions are of the form:
\(
    h_i(p) = r_i - \|p - c_i\|, i\in\{1,\cdots,8\},
\)
where $r_i\in\mathbb{R}_{\geq 0}, c_i \in \mathcal{W},\forall i\in\{1,\cdots,8\}$ (although more general forms can also be chosen). 


\subsection{Unconstrained Inputs}
We first investigate the case of unconstrained inputs, i.e., where the matrix $A_s$ and vector $B_s$ in \eqref{eq:input_constr} are empty, hence only the STL-based constraints are enforced. Still, satisfaction of all of the predicates at each time instance is infeasible, as many time windows overlap {\color{black}(see Eq. \ref{eq:stl:formula})} hence the control scheme of Alg. \ref{alg:control_scheme} is required to compute feasible inputs to \eqref{eq:dynamics_hybrid}. The workspace, STL regions and the robot's trajectory are depicted in Fig. \ref{fig:ws:unsat}. The evolution of the predicates and the system's input are depicted in Fig. \ref{fig:pred:unsat}. Our method's (Alg. \ref{alg:control_scheme}) ability to satisfy all of the STL predicates in the absence of input constraints is verified numerically and can be evidenced in Fig. \ref{fig:pred:unsat}, where the predicates reach a value of $0$ within the respective time windows (shaded areas) for the ``eventually'' predicates and remain positive within the windows for the ``always'' ones. 
\begin{figure}[h]
    \centering
    \includegraphics[trim= 1.5cm 0.cm 4.cm 1.35cm,clip,width=1\linewidth]{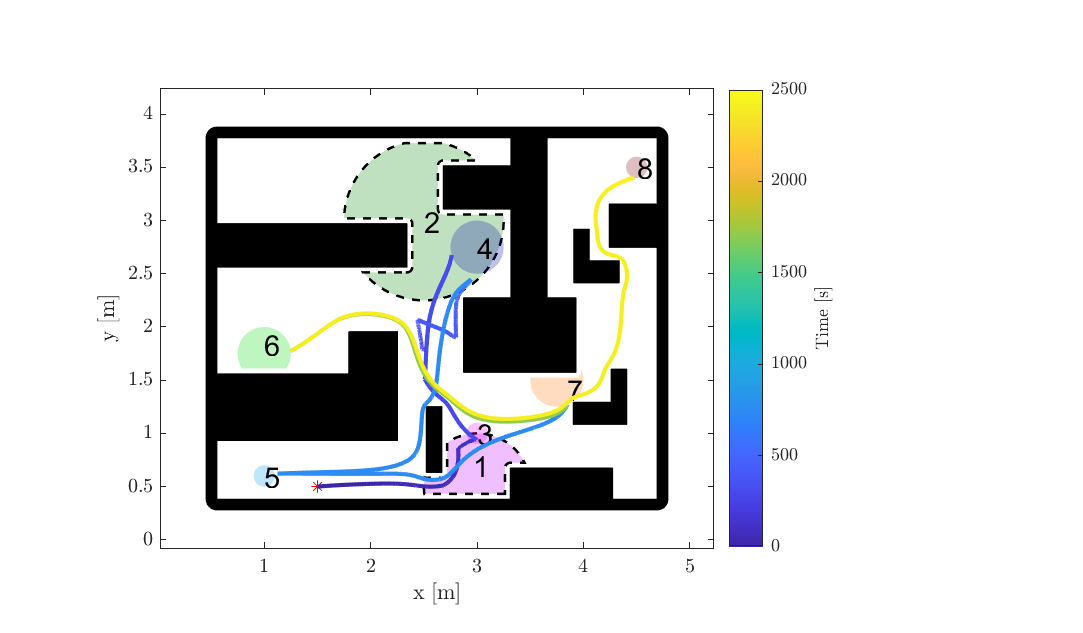}
    \vspace{-0.9cm}
    \caption{{\color{black}\textbf{Proposed Controller.}} Robot workspace (outer boundary and obstacles in black) and numbered STL zero super-level regions (colored regions) of ({\color{black} Eq. \eqref{eq:stl:formula}}) for the unconstrained input case. The robot's trajectory is depicted through the gradient-colored curve, where the color depicts time in seconds.}
    \label{fig:ws:unsat}
\end{figure}


\begin{figure*}[h]
    \centering
    \includegraphics[trim= 1.25cm 0.5cm 1.5cm .5cm,clip,width=0.7\linewidth]{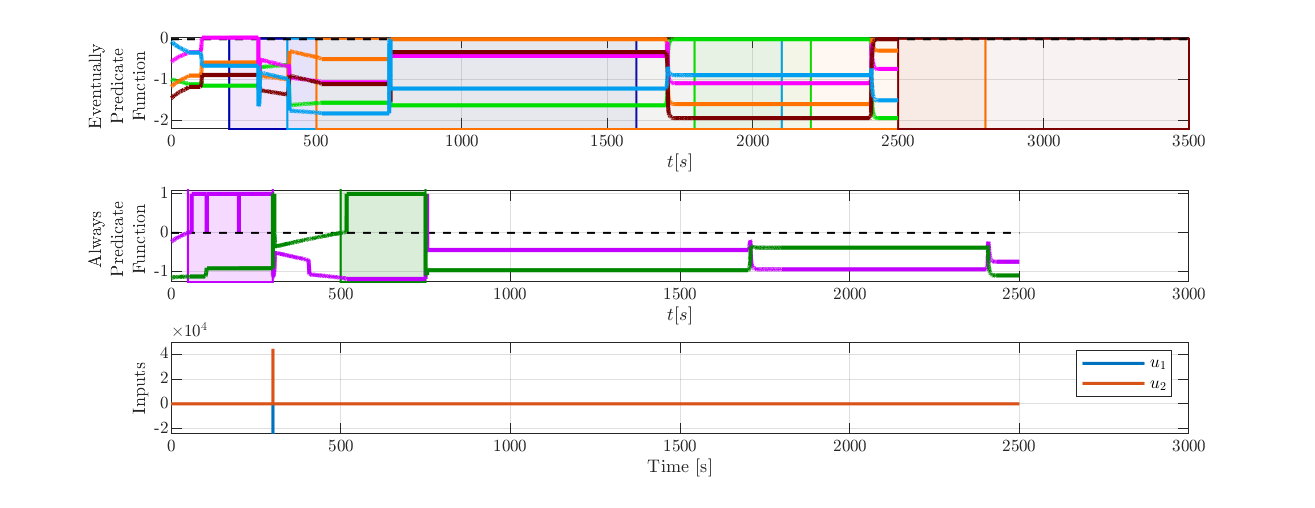}
    \caption{Evolution of the STL predicates for the unconstrained input case of Fig. \ref{fig:ws:unsat}. The eventually (top), always (middle) predicates, as well as the computed input (bottom) are depicted. The colors correspond to the colored regions in Fig. \ref{fig:ws:unsat}, while the coloured regions in the predicate plots depict the time windows for which the predicates are active ({\color{black} Eq. \eqref{eq:stl:formula}}). It is verified that all predicates are satisfied.}
    \label{fig:pred:unsat}
\end{figure*}

\subsection{Constrained Inputs}
We then investigate the case of constraint inputs, more specifically imposing bounds on the input's components $u = \left[u_1,u_2\right]^\top, u_1,u_2 \in \left[ -2, 2\right]$, hence 
$A_s = 
\begin{bmatrix}
    -I_2 & I_2    
\end{bmatrix}$ and 
$B_s = 
\begin{bmatrix}
    -2 & -2 & -2 & -2
\end{bmatrix}^\top$ in \eqref{eq:input_constr}. These input bounds render Prob. 1 harder to solve and compromises the robot's ability to satisfy all of the constraints. The workspace, STL regions and the robot's trajectory are depicted in Fig. \ref{fig:ws:sat}. The evolution of the predicates and the system's input are depicted in Fig. \ref{fig:pred:sat}. Our method is able to satisfy six out of the eight predicates. We underline that the inability of the method to satisfy predicates 2 and 4 does not stem from inherent infeasibility, since our method is not proven to be complete, but rather the homeomorphism to the disk ``distorts'' the geometry of the physical workspace. This is however a meaningful compromise in the context of non-convex workspaces, where direct CBF design on the robot's workspace is difficult. Additionally, the choice of STL-CBF parameters, e.g., $\delta$ in \eqref{eq:eventually_cbf}, \eqref{eq:always_cbf}, the choice of $z$ in \eqref{eq:always_cbf} impacts the CBF conditions in \eqref{eq:A_matrix}, \eqref{eq:B_matrix}. In future works we aim at investigating adaptive techniques for the aforementioned parameters that may render versions of the proposed methodology complete in the face of conflicting specifications. Furthermore, in Figs. \ref{fig:comp_ctr:sat} and \ref{fig:comp_srch:sat} the computational times for solving the CBF-QP \eqref{eq:qp_con*} and Alg. \ref{alg:conf_search} are depicted respectively. Even in the presence of 8 constraints, where $|\mathcal{C}|= 2^8 = 256$, the search takes a mean time of $\approx 0.1[s]$. 
\begin{figure}[h]
    \centering
    \includegraphics[trim= 2.0cm 0.5cm 3cm 1.0cm,clip,width=0.9\linewidth]{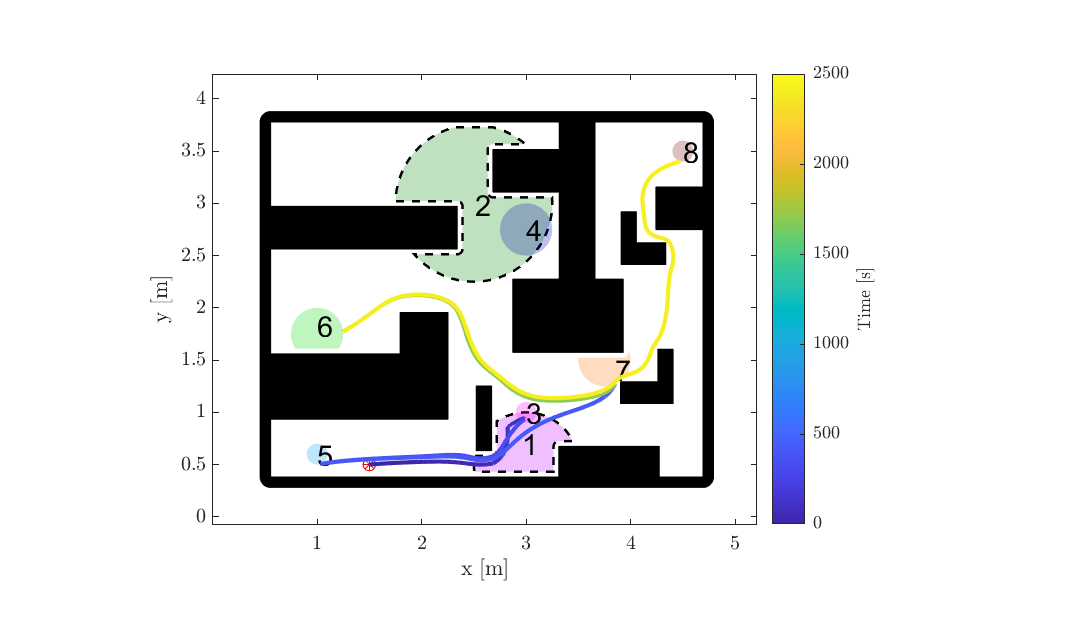}
    \vspace{-0.5cm}
    \caption{{\color{black}\textbf{Proposed Controller.}} Robot workspace (outer boundary and obstacles in black) and numbered STL zero super-level regions (colored regions) {\color{black}of Eq. \eqref{eq:stl:formula}} for the constrained input case. The trajectory is depicted through the gradient-colored curve, where the color depicts time $[s]$.}
    \label{fig:ws:sat}
\end{figure}


\begin{figure*}[h]
    \centering
    \includegraphics[trim=  1.7cm 0.5cm 1.5cm .5cm,clip,width=0.75\linewidth]{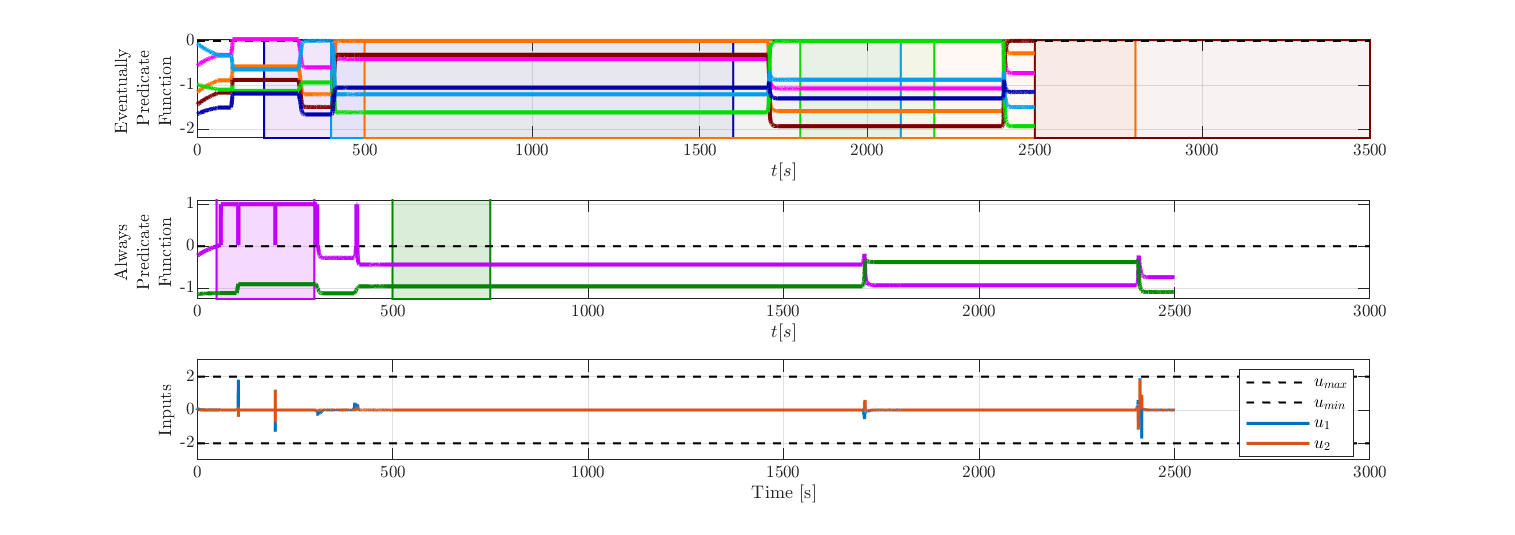}
    \vspace{-0.25cm}
    \caption{Evolution of the STL predicates for the constrained input case of Fig. \ref{fig:ws:sat}. The eventually (top), always (middle) predicates, as well as the computed input (bottom) are depicted. The colors correspond to the colored regions in Fig. \ref{fig:ws:sat}, while the   regions in the predicate plots depict the time windows for which the predicates are active ({\color{black} Eq. \eqref{eq:stl:formula}}). The predicates 1,3,5,6,7,8 are satisfied, while predicates 2 and 4 are unsatisfied.}
    \label{fig:pred:sat}
\end{figure*}

\begin{figure}[h]
    \centering
    \includegraphics[trim=  0.75cm 0.0cm 1.2cm 0.25cm,clip,width=1\linewidth]{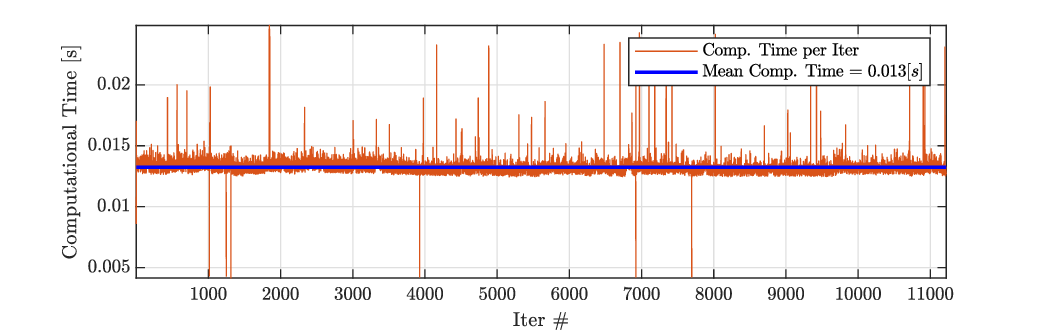}
    \vspace{-0.57cm}
    \caption{Computational times for the CBF-QP controller \eqref{eq:qp_con*}, case with constrained inputs.}
    \label{fig:comp_ctr:sat}
\end{figure}
\begin{figure}[h]
    \centering
    \includegraphics[trim= 1.0cm 0.0cm 1.25cm 0.cm,clip,width=1\linewidth]{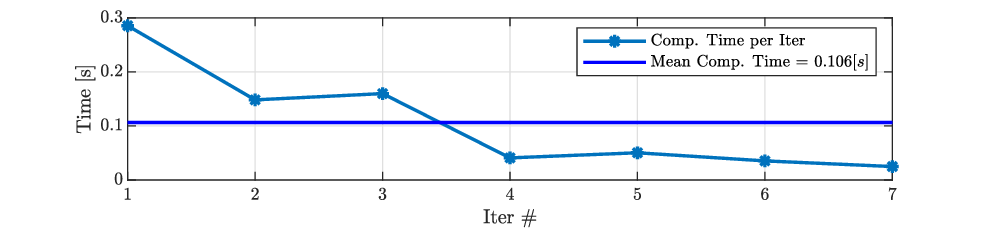}
    \vspace{-0.5cm}
    \caption{Computational times for {\color{black}a serial implementation of} {\fontfamily{qcr}\selectfont Configuration Search}  --Alg. \ref{alg:conf_search}-- for the constrained input case. The algorithm was called 7 times {\color{black}(due to detecting infeasibility)} during the simulation of the robotic task.}
    \label{fig:comp_srch:sat}
\end{figure}

\subsection{Comparative Simulations}
\begin{figure}[h]
    \centering
    \includegraphics[trim= 1.5cm 0.2cm 1.7cm 0.5cm,clip,width=0.8\linewidth]{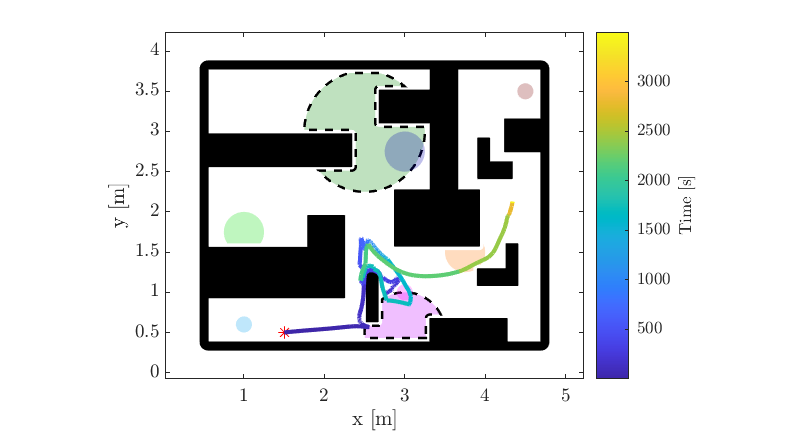}
    \vspace{-0.25cm}
    \caption{{\color{black}\textbf{Benchmark Controller.}} Robot workspace (outer boundary and obstacles in black) and  STL zero super-level regions (colored regions) of ({\color{black} Eq. \eqref{eq:stl:formula}}) for the constrained input case with the method of variable relaxation. The trajectory is depicted through the gradient-colored curve, where the color depicts time.}
    \label{fig:ws:sat:relax}
\end{figure}
\begin{figure}[h]
    \centering
    \includegraphics[trim=  1.05cm 0.0cm 1.5cm 0.25cm,clip,width=1\linewidth]{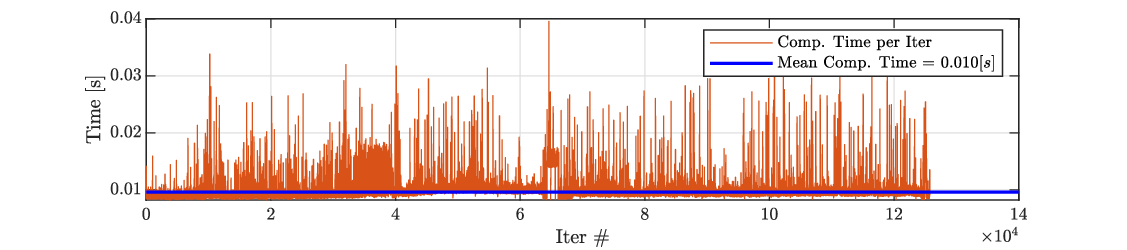}
    \caption{Computational times for the CBF-QP controller with the relaxation variables for the constraints. }
    \label{fig:comp_ctr:sat:relax}
\end{figure}
Our simulation studies conclude with the comparison with a method for relaxing the constraints through relaxation variables on the STL-CBFs, similar to \cite{Chinneck02102019}, depicted in Fig. \ref{fig:ws:sat:relax}, with the corresponding CBF-QP solution times depicted in Fig. \ref{fig:comp_ctr:sat:relax}. Since all of the constraints are considered at once, relaxing the constraints fails to satisfy all predicates besides \#7. This highlights the main advantage of the proposed method, namely that modeling constraint enforcement via the discrete configuration variable, enables ensuring that some STL-CBF constraints will be satisfied, even if the proposed methodology is not complete. At the same time, the proposed method exhibits {\color{black}similar, but slightly increased} computational times for input computation {\color{black} due to having to perform the LP feasibility evaluation online} (see Fig. \ref{fig:comp_ctr:sat:relax} compared to Fig. \ref{fig:comp_ctr:sat}), with the caveat of having to perform the configuration search algorithm.    

\section{Conclusions \& Future Work}
In this work, a novel task and motion planning framework based on STL-CBFs is proposed. Importantly, the methodology tackles arbitrarily complex planar workspaces, while at the same time handling multiple, overlapping tasks, enabled by hybrid control. Therefore, this approach consists a step towards bridging the gap between planning and control. However, investigating more general CBF formulations as well as horizon-planning approaches is necessary to ensure completeness of the method. Furthermore, in the future we aim at considering more complex STL predicates and logical operations, such as disjunctions and nested predicates, which require significant additional effort. {\color{black}Finally, developing mechanisms for predicting and amending infeasibility in face of overlapping tasks and input constraints is a research direction we aim at pursuing.} 

\bibliography{ifacconf}             
\end{document}